\documentclass{article}

\usepackage{booktabs}
\usepackage{wrapfig}
\usepackage{amsmath,amssymb,amsthm,amsfonts}
\usepackage{graphicx}
\usepackage{hyperref}
\usepackage[dvipsnames]{xcolor}
\usepackage{algorithm}
\usepackage{algpseudocode}
\usepackage{float}
\usepackage{subcaption}
\usepackage{todonotes}
\usepackage{setspace}

\newtheorem{proposition}{Proposition}
\newtheorem{fact}{Fact}
\newtheorem{theorem}{Theorem}
\newtheorem{theorem*}{Theorem}

\newtheorem{lemma}{Lemma}

\newtheorem{definition}{Definition}
\newtheorem{remark}{Remark}
\usepackage{csquotes}

\newcommand{\D}{\mathrm{D}}
\newcommand{\bc}{\boldsymbol{\theta}_c}
\newcommand{\ec}{\boldsymbol{\eta}_c}

\usepackage[preprint]{corl_2025} 

\title{Bregman Centroid Guided Cross-Entropy Method}

\author{
    Yuliang Gu\textsuperscript{1}, 
    Hongpeng Cao\textsuperscript{2}, 
    Marco Caccamo\textsuperscript{2}, 
    Naira Hovakimyan\textsuperscript{1} \\
    \textsuperscript{1}Department of Mechanical Science and Engineering, UIUC, United States \\
    \textsuperscript{2}School of Engineering and Design, TUM, Germany
}

\begin{document}
\maketitle
\begin{abstract}
The Cross-Entropy Method (CEM) is a widely adopted trajectory optimizer in model-based reinforcement learning (MBRL), but its unimodal sampling strategy often leads to premature convergence in multimodal landscapes. In this work, we propose \textbf{$\mathcal B$regman-$\mathcal C$entroid Guided CEM ($\mathcal{BC}$-EvoCEM)}, a lightweight enhancement to ensemble CEM that leverages \emph{Bregman centroids} for principled information aggregation and diversity control. \textbf{$\mathcal{BC}$-EvoCEM} computes a performance-weighted Bregman centroid across CEM workers and updates the least contributing ones by sampling within a trust region around the centroid. Leveraging the duality between Bregman divergences and exponential family distributions, we show that \textbf{$\mathcal{BC}$-EvoCEM} integrates seamlessly into standard CEM pipelines with negligible overhead. Empirical results on synthetic benchmarks, a cluttered navigation task, and full MBRL pipelines demonstrate that \textbf{$\mathcal{BC}$-EvoCEM} enhances both convergence and solution quality, providing a simple yet effective upgrade for CEM.
\end{abstract}


\section{Introduction}\label{sec:intro}

The \emph{Cross–Entropy Method} (CEM) is a derivative–free stochastic
optimizer that converts an optimization problem into a sequence of
rare event estimation tasks~\citep{rubinstein2004cross,de2005tutorial}. At each iteration, CEM samples $N$ candidates \(\{x_j\}_{j=1}^{N}\) from a parametric distribution $p_{\theta_t}$, selects top lowest-cost samples as an elite set \(\mathcal{E}_t\), and updates the parameters by maximizing the log-likelihood of these elites:
\begin{equation}\label{eq:CEM}
\theta_{t+1} \;=\;
   \arg\max_{\theta}\sum_{x\in\mathcal{E}_t}\log p_{\theta_t 
   }(x),    
\end{equation}
optionally smoothed via exponential averaging for stability. Its reliance solely on cost-based ranking instead of gradient information has made CEM a widely adopted solver for high-dimensional, nonconvex optimization tasks in robotics and control~\cite{pinneri2021sample,kobilarov2012cross,banks2020multi}.

\paragraph{CEM in MBRL} In model–based reinforcement learning (MBRL), an agent learns a
predictive model of the environment and plans through that model to reduce costly real‐world interactions~\cite{ha2018world, nagabandi2018neural, silver2017mastering}.  Stochastic
model predictive control (MPC) is a widely used planning strategy in this setting~\cite{williams2016aggressive, okada2020variational, chua2018deep, zhang2022simple, deisenroth2011pilco}.  At every decision step, MPC solves a finite–horizon trajectory
optimization problem, executes only the first action,
observes the next state, and replans. The CEM is often chosen as the optimizer within this loop due to its simplicity, reliance solely on cost function evaluations, and robustness to noisy or nonconvex objectives.

Despite these advantages, vanilla CEM suffers from its inherent
\emph{mode–seeking} nature: as the elites concentrate, it often collapses the search into a local optimum, which significantly limits the exploration in complex multimodal landscapes typical of RL tasks. Ensemble strategies have been proposed to mitigate this issue by running multiple CEM workers. \textbf{Centralized ensembles} merge the elite sets of all workers
and fit an explicit mixture model (e.g., commonly a Gaussian
mixture~\citep{okada2020variational}).  Although more expressive, they introduce additional hyperparameters (number of components, importance weights) and increase computational cost due to joint expectation maximization (EM) steps. \textbf{Decentralized ensembles} run multiple CEM instances in parallel, keep them independent, and output the best solution at termination~\cite{zhang2022simple}.  This approach is simple and scalable but tends to duplicate exploration effort and may reach premature consensus if poorly initialized.

\paragraph{Our Approach.} Motivated by the trade-off between diversity preservation and
computational efficiency, we introduce \textbf{$\mathcal B$regman-$\mathcal C$entroid Guided CEM ($\mathcal{BC}$-EvoCEM)}, a hybrid strategy that retains the
independent updates of decentralized ensembles yet introduces a simple information–geometric coupling across workers. At each CEM iteration, \textbf{$\mathcal{BC}$-EvoCEM} computes a \emph{performance–weighted
Bregman centroid}~\cite{nielsen2009sided} of all workers’ distributions. The centroid then defines
both a reference point and a \emph{Bregman ball} trust region.
Any worker whose distribution lies too close to the centroid or exhibits high cost is respawned by drawing new parameters from this trust region (see Fig.~\ref{fig:demo}).

\paragraph{Contributions.} 1) We formulate an information–geometric aggregation rule based on Bregman centroids that summarizes ensemble CEM workers with \emph{negligible} computation cost. 2) We provide a lightweight integration into the MPC loop for MBRL, preserving the benefits of \textbf{$\mathcal{BC}$-EvoCEM} with the simplicity of a standard warm-start heuristic. 3) Through experiments on multimodal synthetic functions, cluttered navigation tasks, and full MBRL benchmarks, we demonstrate faster convergence with improved performance relative to the vanilla and decentralized CEM.

\begin{figure}[t]
    \centering
    \begin{subfigure}[t]{0.47\textwidth}
        \centering
        \includegraphics[width=0.8\linewidth]{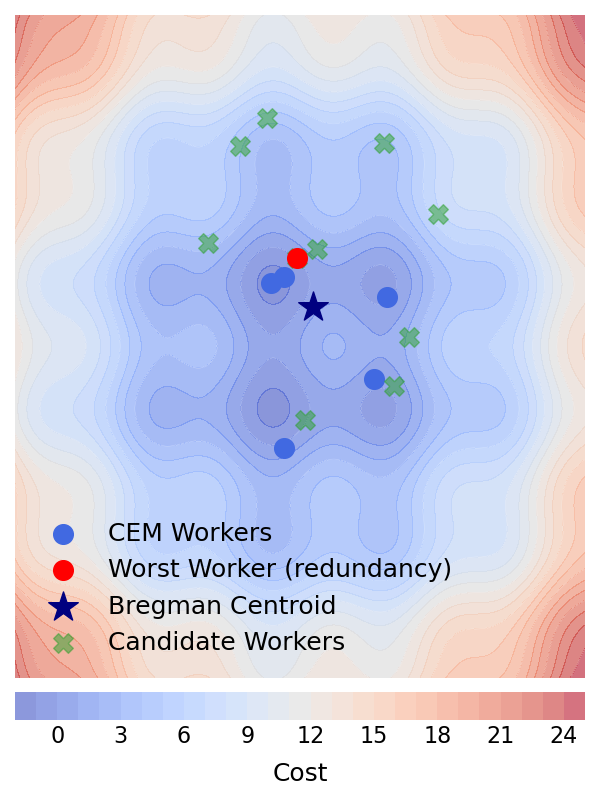} 
    \end{subfigure}
    \hfill
    \begin{subfigure}[t]{0.47\textwidth}
        \centering
        \includegraphics[width=0.8\linewidth]{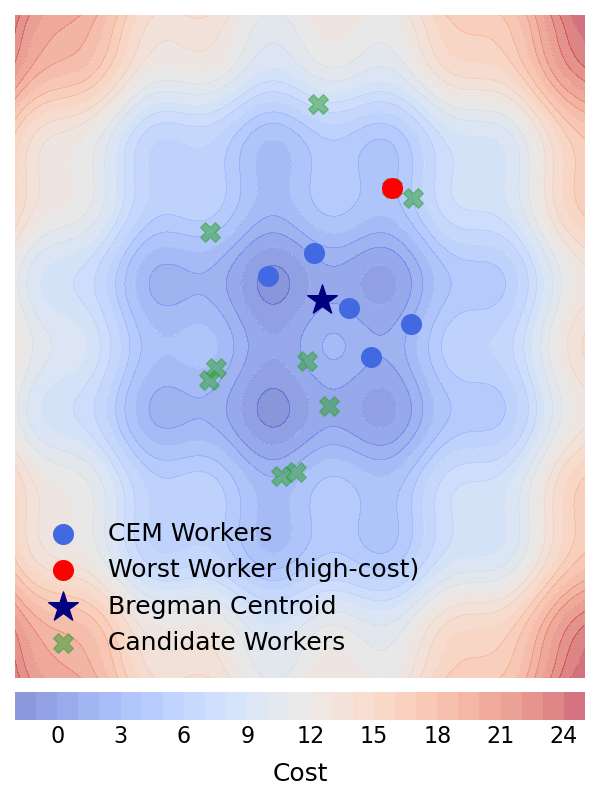} 
    \end{subfigure}
\caption{Illustration of $\mathcal{BC}$-EvoCEM (only \emph{means} are shown). \textcolor{Blue}{\(\bullet\) Active CEM workers}. \textcolor{red}{\(\bullet\) Worst workers} identified due to redundancy(left) and poor quality(right). \textcolor{BlueViolet}{\(\bigstar\) Bregman Centroid} as a geometric average of active workers. \textcolor{Green}{\(\boldsymbol{\times}\) Potential Candidates} sampled from the trust-region.}
\label{fig:demo}
\vspace{-5mm}
\end{figure}

\section{Bregman Divergence}
We review key definitions and properties of Bregman divergences~\cite{bregman1967relaxation} used throughout this paper. Let \(F:\mathcal{S}\to\mathbb{R}\) be a strictly convex, differentiable potential function on a convex set \(\mathcal{S}\). The \emph{Bregman divergence} between any two points \(x,y\in\mathcal{S}\) is defined as
\begin{equation}\label{eq:bregman-def}
\D_F(x\Vert y)
\;=\;
F(x)\;-\;F(y)\;-\;\bigl\langle x - y,\;\nabla F(y)\bigr\rangle,
\end{equation}
where \(\langle\cdot,\cdot\rangle\) denotes the inner product. Although \(\D_F\) is not a \emph{metric}, it retains ``distance‐like'' and statistical properties useful in optimization and machine learning applications~\cite{snell2017prototypical,ahn2021efficient}. In particular, its bijective correspondence with exponential families provides a range of clustering and mixture modeling techniques~\cite{banerjee2005clustering}.

\begin{wrapfigure}{r}{0.35\textwidth}
\vspace{-3mm}
\centering
\includegraphics[width=0.95\linewidth]{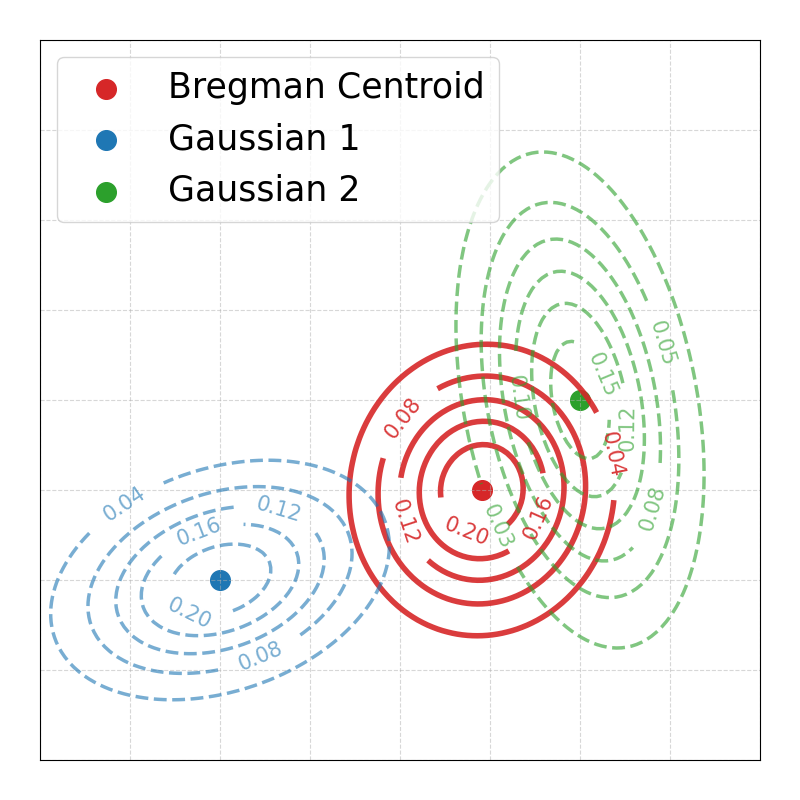} 
\caption{Illustration of the Bregman centroid of two Gaussians.}
\label{fig:bc_demo}
\vspace{-5mm}
\end{wrapfigure}

\paragraph{Bregman Centroid \& Information Radius.} Given a collection of points \(\{x_i\}_{i=1}^n\subset\mathcal{S}\), the \emph{Bregman centroid} (right‐sided) is the solution to the following minimization problem~\cite{nielsen2009sided}:
\[
\boldsymbol{x_c}
\;=\;
\arg\min_{x\in\mathcal{S}}
\frac{1}{n}\sum_{i=1}^n \D_F(x_i\Vert x).
\]
The corresponding minimized value is known as the \emph{Information Radius} (IR)~\cite{csiszar2004information} (\emph{Bregman Information} in ~\cite{banerjee2005clustering}), which characterizes the \emph{diversity} of the set \(\{x_i\}\) under the geometry induced by \(F\). Notably, for $F=\|x\|^2$, the IR coincides with the sample variance of the set. See~\cite{banerjee2005clustering,nielsen2009sided,csiszar2004information} for more details.

\section{Method}
    We propose a statistical characterization of a set of distributions through a \textbf{weighted Bregman centroid}~\cite{nielsen2009sided}. Let a set of CEM distributions \(\{\theta_1, \dots, \theta_n\}\)be  drawn from \emph{a parametric family} \(\{p_{\theta}\}_{\theta\in\Theta}\), each associated with an importance weight \(w_i\) that reflects its solution quality. Formally, the weighted Bregman centroid \(\bc\) of the set is defined as
\begin{equation}\label{eq:w-BC} 
    \bc
    \;=\; 
    \arg\min_{\theta \in \Theta} \;
    \sum_{i=1}^n 
    w_i \,\D_F\bigl(\theta_i \,\big\|\, \theta\bigr),
\end{equation} 
where \(\D_F\) is a Bregman divergence associated with a potential \(F\). During the CEM iterations, the weights are assigned based on the performance of \(p_{\theta_i}\) (e.g.,
\(
    w_i 
    \;\propto\; 
    \exp(-\,\mathbb{E}_{p_{\theta_i}}[J(x)])
\)).
Hence, the centroid serves as a \emph{performance-weighted ``geometric average''} of all CEM workers. To ensure that the ensemble remains effective in terms of both \textbf{performance and diversity}, we introduce two essential definitions:

\begin{definition}[Relevance Score]
Let the \emph{score} of a worker $\theta_i$ be the weighted Bregman divergence to the centroid: 
\(\boxed{
    \gamma_i \;=\; w_i \,\D_F\bigl(\theta_i \,\big\|\,\bc\bigr).}
\)
\end{definition}
\vspace{2mm}
\begin{definition}[Trust Region]
For \(\Delta >0\), the trust region is defined as a \textbf{Bregman ball} centering at $\bc$ with radius $\Delta$:
\(\boxed{
    \mathcal{B}_{\Delta}\bigl(\bc\bigr)
    \;=\;
    \Bigl\{
        \theta \in \Theta
        \;\Bigm|\;
        \D_F\bigl(\theta\,\big\|\, \bc\bigr)
        \;\le\;
        \Delta\}
    \Bigr\}.}
\)
\end{definition}

\paragraph{Interpretation of Relevance Scores.}
Intuitively, a low relevance score 
\(\gamma_i\) 
indicates that either the worker’s performance weight \(w_i\) is low or it's close to the centroid \(\bc\). Such workers contribute minimally to both exploitation and exploration and are therefore candidates for replacement. Moreover, one can verify that \(\gamma_i\) is exactly $\theta_i$'s contribution to the Information Radius (IR) of the set under the probabilistic vector \(\mathbf{w}\).

\paragraph{Role of the Trust Region.}
The trust region 
\(\mathcal{B}_{\Delta}(\bc)\) 
constrains where new workers may be introduced, ensuring that they remain in average proximity to the active workers. Crucially, defining the trust region as a \emph{Bregman ball} aligns with the intrinsic geometry of the chosen parametric family, which offers theoretical insights and computational advantages when employing exponential families, as further discussed in Section~\ref{sec:stoch_opt}.

\paragraph{Bregman Centroid Guided Evolution Strategy.}
Building on the above components, we propose a simple \emph{Bregman Centroid Guided Evolution Strategy} for ensemble CEM (see Alg.~\ref{alg:guided_evolution}). At each iteration, it begins with distributed CEM updates, continues with score evaluation, and finishes with an evolutionary update that replaces the lowest-scoring worker with a condidate sampled from the trust region.

\begin{algorithm}[H]
\small
\setstretch{1.3}
\caption{$\mathcal{BC}$-evoCEM: Guided Evolution Strategy For CEM}
\label{alg:guided_evolution}
\begin{algorithmic}[1]
\Require CEM distributions $\{\theta_i\}_{i=1}^n$, cost function $J(\,\cdot\,)$, iterations $T$
\For{$t=1$ \textbf{to} $T$}
    \State CEM update: $\{\theta_i\}\leftarrow\textsc{Distributed CEM}(\{\theta_i\},J)$
    \State Performance weights: $w_i\propto\exp(- \mathbb{E}_{p_{\theta_i}}[J\;])$
    \State Centroid: $\bc\leftarrow\arg\min_{\theta}\sum_i w_i \D_F(\theta_i\|\theta)$
    \State Scores: $\gamma_i\leftarrow w_i\, \D_F(\theta_i\|\bc)$
    \State Replace: $\theta_{\min}\leftarrow\arg\min_i \gamma_i$;\;
           $\theta_{\min}\leftarrow\textsc{Sample}\bigl(\mathcal{B}_{\Delta}(\bc)\bigr)$
\EndFor
\State \textbf{Return} centroid $\bc$ and optimized workers $\{\theta_i\}_{i=1}^n$ 
\end{algorithmic}
\end{algorithm}

\section{Stochastic Optimization in Exponential Families}
\label{sec:stoch_opt}
In this section, we leverage the relationship between regular exponential families and Bregman divergences to gain statistical insight into our guided evolution strategy (Alg.~\ref{alg:guided_evolution}) and achieve substantial computational savings in the CEM pipeline. 

Let \(\{p_{\theta}\}_{\theta\in\Theta}\) be a \emph{regular, minimal} exponential family in \emph{natural} form
\[
p_{\theta}(x)\;\propto\;\exp\bigl\{\theta^{\top}T(x)-\Psi(\theta)\bigr\}, 
\quad
\theta\in\Theta\subset\mathbb{R}^d,
\]
where \(T(x)\in\mathbb{R}^d\) represents the sufficient statistics and \(\Psi\) is the strictly convex cumulant. The corresponding Bregman divergence is \(\D_{\Psi}(\theta\Vert\theta')\)~\cite{banerjee2005clustering}. We denote the \emph{mean parameter} by \(
\eta=\nabla\Psi(\theta)=\mathbb{E}_{p_{\theta}}[T(X)].\)
Since the map \(\nabla\Psi:\Theta\to\mathcal{E}=\mathrm{int}\,\nabla\Psi(\Theta)\) is a bijection between the natural space \(\Theta\) and the mean space \(\mathcal{E}\) (see~\cite{barndorff2014information,amari1995information}), we advocate representing and manipulating distributions in the mean space.

As the core of our method, the Bregman centroid admits a simple form in mean coordinates:
\begin{proposition}[Centroid in mean coordinates]\label{prop:dual_centroid}
Given weights \(\mathbf w=(w_1,\dots,w_n)\), \(w_i\ge0\), \(\sum_i w_i=1\), and corresponding mean parameters \(\eta_i\), the weighted Bregman centroid satisfies
\[
\ec \;=\;\sum_{i=1}^n w_i\,\eta_i,
\quad
\bc \;=\;(\nabla\Psi)^{-1}(\ec).
\]
\end{proposition}
\vspace{-5mm}
\begin{proof}
By the \emph{mean‐as‐minimizer} property of right‐sided Bregman divergences~\cite{banerjee2005clustering,nielsen2009sided}, the optimality condition \(\nabla\Psi(\bc)=\ec\) uniquely determines \(\bc\) via the bijection \(\nabla\Psi\colon\Theta\to\mathcal{E}\).
\end{proof}
Given that CEM’s likelihood evaluations (see Eq.~\eqref{eq:CEM}) already yield the empirical mean
\(
\widehat\eta_i = \frac{1}{N}\sum_{j=1}^N T_i(x),
\)
the centroid
\(\ec\) is obtained \textbf{\emph{for free}}. No extra optimization (e.g., solving~\eqref{eq:w-BC}) is required.  
\subsection{Scoring as Likelihood-based Ranking}
\label{sec:score_ranking}
Since only \emph{relative} scores matter for ranking workers in Alg.~\ref{alg:guided_evolution}, we may drop all terms independent of \(i\) and rewrite the \emph{relevance score} as (see Appendix~\ref{appendix:scores})
\[
   \gamma_i
   =\;
   w_i\, \D_{\Psi}\!\bigl(\theta_i\parallel\bc\bigr)
   \;\;\propto\;\;
   w_i\Bigl[\Psi(\theta_i)-\langle\theta_i, \ec\rangle\Bigr] \;=\;
-\,w_i\,\ell(\theta_i;\,\ec),
\]
where \(
\ell(\theta;x)
=\langle\theta,x\rangle-\Psi(\theta)
\) is precisely the \emph{per-sample} log-likelihood that the natural parameter $\theta_i$ would attain under some (hypothetical) dataset whose empirical average is $\ec$. Intuitively, ranking workers by $\ell(\theta_i;\,\ec)$ is equivalent to asking:
\begin{displayquote}
\emph{How well the worker $\theta_i$ explain the aggregated information collected from all workers $\{\theta_i\}_{i=1}^n$?}
\end{displayquote}
This yields a cheap moment-matching ranking metric that requires only inner products and evaluations of \(\Psi\).

\subsection{Efficient Trust‑Region Sampling}
\label{sec:ball_sampling}
Given the centroid $\ec$, we sample candidates from the trust region
\(
   \mathcal{B}_{\Delta}(\bc)
\)
by working with its dual characterization
\(
   \mathcal{S}
   := \{\eta\in\mathcal E : \D_{\Psi^{\!*}}(\ec\parallel\eta)\le\Delta\},
\) where \(\Psi^{\!*}\) is the convex conjugate of \(\Psi\).
\begin{definition}[Radial Bregman Divergence]\label{def:rDiv}
For \(v\in\mathbb S^{d-1}\)(the unit sphere) define the radial Bregman divergence with respect to a fixed $\eta \in \mathcal E$
\[
     g_v(\rho)
     := \D_{\Psi^{\!*}}\!\bigl(\eta\parallel\eta+\rho v\bigr),
     \qquad \rho\ge0.
\]
\end{definition}

\begin{theorem}\label{thm:sampling_valid}
The \textsc{Turst-Region Sampler} (see Alg.~\ref{alg:exact_sampling}) produces
\(\eta_{\mathrm{new}}\sim\mathrm{Unif}(\mathcal{S})\)
and \(\theta_{\mathrm{new}}\in\mathcal{B}_{\Delta}(\bc)\).
If \(\Psi\) is quadratic (e.g., fixed-$\Sigma$ Gaussian),
\(\theta_{\mathrm{new}}\) is uniformly distributed in
\(\mathcal{B}_{\Delta}(\bc)\).
\end{theorem}
\begin{proof}
See Appendix~\ref{appendix:proof}.
\end{proof}

\vspace{-4mm}
\begin{figure}[H]
\begin{minipage}{0.45\textwidth}
\begin{algorithm}[H]
\small
\setstretch{1.2}
\centering
\caption{\textsc{Trust-Region Sampler}}
\label{alg:exact_sampling}
\begin{algorithmic}[1]
\Require centroid $\ec$, radius $\Delta$ 
\Require radial divergence $g_v(\rho)$
\State Draw $v\sim \mathrm{Unif}(\mathbb S^{d-1})$ 
\State Root-solve \ $g_v(\rho_{\max})=\Delta$ 
\State Sample $u\sim\mathrm{Unif}[0,1]$ 
\State \textbf{Return} $\eta_{\text{new}}\gets\ec+u^{1/d}\,\rho_{\max}(v)\,v$
\end{algorithmic}
\end{algorithm}
\end{minipage}
\hfill
\begin{minipage}{0.45\textwidth}
\begin{algorithm}[H]
\small
\setstretch{1.2}
\centering
\caption{\textsc{Proxy Sampler}}
\label{alg:proxy}
\begin{algorithmic}[1]
\Require centroid $\ec$, radius $\Delta$
\Require Hessian $\mathrm H=\nabla^2\Psi^*(\ec)$
\State Draw $v\!\sim\!\mathrm{Unif}(\mathbb S^{d-1})$
\State Set $\widehat\rho_{\max}(v)\gets\sqrt{2\Delta/(v^\top \mathrm H v)}$
\State Sample $t\sim\mathrm{Unif}[-\widehat\rho_{\max},\widehat\rho_{\max}]$
\State \textbf{Return} $\eta_{\text{new}}\gets\ec+t\,v$
\end{algorithmic}
\end{algorithm}
\end{minipage}
\end{figure}

\vspace{-3mm}
\paragraph{Local Proxy Sampling \& Gaussian Case.}
While Alg.~\ref{alg:exact_sampling} is general and exact, every draw incurs a root–solving $g_v(\rho_{\max}) =\Delta$, which is expensive for high-dimensional parameterization (e.g., action sequences in MBRL). In practice, we \emph{locally} approximate \(g_v(\rho) \approx \frac{1}{2} \rho^2 v^\top \mathrm H v\) at the centroid \(\ec\) up to second order, where 
\(
  \mathrm{H}\;=\;\nabla^{2}\Psi^{*}(\ec).
\)
This yields an \emph{ellipsoidal} trust region \(\widehat{\mathcal S}\) with a \textbf{closed-form} maximal radius
\[
  \widehat{\mathcal S}
  \;=\;
  \bigl\{\eta:\;
        (\eta-\ec)^{\!\top} \mathrm H (\eta-\ec)\le 2\Delta\bigr\}\;\; \Longrightarrow\;\;
        \widehat\rho_{\max}(v)
  \;=\;
  \sqrt{2\Delta\big/ (v^\top \mathrm{H}v)},
\]
which requires only one dot product and a square root (See the \textsc{Proxy Sampler} in Alg.~\ref{alg:proxy}).

For diagonal Gaussian action-sequence planner (common in MBRL~\cite{chua2018deep, zhang2022simple,deisenroth2011pilco}), the Hessian \(\mathrm{H} = \mathrm{diag}(h_1, \dots, h_d)\) itself is diagonal and the resulting trust region becomes axis-aligned with principal radii \(\sqrt{2\Delta / h_i}\). In such cases, sampling further reduces to simple coordinate-wise operations (see Appendix~\ref{appendix:proxy_sampling} for details).

\begin{table}[H]
\centering
\small
\setlength{\tabcolsep}{5pt}
\caption{Major operations in the CEM loop.  Here \(n\) is the number of workers, \(d\) the parameter dimension, and \(m\) the (typically small) number of iterations in the root solver of Alg.~\ref{alg:exact_sampling}.  All cost are worst-case. Quantities marked $\dagger$ are already computed in CEM.}
\label{tab:cem_costs}
\begin{tabular}{@{}lll@{}}
\toprule
\textbf{Operation} & \textbf{Cost} & \textbf{Comment}\\
\midrule
Centroid  & \(O(nd)\) & Weighted average of $\eta_i^{\dagger}$\\
Relevance score\;($\gamma_i$) & \(O(d)\) per worker & One inner product $+\;\Psi(\theta_i)$ (closed form)\\
Exact sampler ($d\!\lesssim\!100$) & \(O(d+m)\) & Root solve for $\rho_{\max}$ (e.g., secant method)\\
Proxy sampler ($d\!\gg\!100$) & \(O(d)\) & Closed-form $\widehat \rho_{\max}$\\
\bottomrule
\end{tabular}
\end{table}

\paragraph{Summary.} By operating with the mean parameterization of the exponential family, \textbf{$\mathcal{BC}$-EvoCEM}'s add-on operations incur negligible overhead. The empirical means required for the Bregman centroid are available from the CEM log–likelihood computation.  All subsequent steps (see Table~\ref{tab:cem_costs}) scale linearly with the parameter dimension and remain trivial compared to environment roll–outs.
Moreover, the geometric interpretation of our method is remarkably intuitive and \emph{Euclidean‐like}: the Bregman centroid coincides with a weighted arithmetic mean, and the proxy trust region resembles an ellipsoidal neighbourhood under a natural affine transformation. 
\clearpage

\section{Bregman Centroid Guided MPC}
\label{sec:mpc_bregman}
The proposed \textbf{$\mathcal{BC}$-EvoCEM} integrates elegantly into the MPC pipeline for MBRL, where the trajectory optimization is performed iteratively in a receding horizon fashion. Instead of warm starting CEM optimizer at time~$t+1$ by shifting the previous solution~\cite{chua2018deep} or restarting from scratch, we use the performance-weighted Bregman centroid of the $K$ independent CEM solutions to initialize the next iteration. To prevent ensemble collapse (i.e., $\mathrm{IR}\to0$), we periodically replace the least-contributing workers by candidates sampled from the trust region.

\begin{wrapfigure}{R}{0.53\textwidth}
\vspace{-4mm}
\begin{minipage}{0.5\textwidth}
\begin{algorithm}[H]
\setstretch{1.2}
\caption{(schematic) Drop-in MPC Wrapper for MBRL.}
\label{alg:bc_mpc_schematic}
\begin{algorithmic}[1]
  \Require $K$ CEM workers, buffer $\mathcal{D}$
  \For{each training iteration}
  \State Train dynamics model $\tilde f$ on 
  $\mathcal{D}$
  \State Initialize Bregman Centroid
  \For{each control step $t = 1,\dots,H$}
    \State Warm start CEM workers by BC
    \State Rollouts by $\tilde f$ \& Update CEM workers
    \State Compute Bregman Centroid
    \State (periodic) Score \& Replace
    \State Execute the best worker's 1st action
    \State Add transitions to $\mathcal{D}$
  \EndFor
  \EndFor
\end{algorithmic}
\end{algorithm}
\end{minipage}
\vspace{-8mm}
\end{wrapfigure}

Building on the stochastic optimization techniques in Sec.~\ref{sec:stoch_opt}, we implement this strategy as a \emph{drop-in MPC wrapper for MBRL} (see Alg.~\ref{alg:bc_mpc_schematic}) that \textbf{1)} preserves the internal CEM update unchanged, \textbf{2)} enforces performance–diversity control via the Bregman centroid, \textbf{3)} and incurs only a few additional vector operations per control step.

Here, the Bregman centroid encapsulates the CEM ensemble’s consensus on promising action-sequences while implicitly encoding optimality-related uncertainties in the warm start. The trust-region based replacement then reinjects diversity in regions where the model is confident. Therefore, this implementation delivers the benefits of guided evolution with the simplicity of a standard warm-start heuristic.

\section{Experimental Results}
\label{sec:result}

\subsection{Motivational Example}
\begin{wrapfigure}{l}{0.5\textwidth}
\centering
\includegraphics[width=0.97\linewidth]{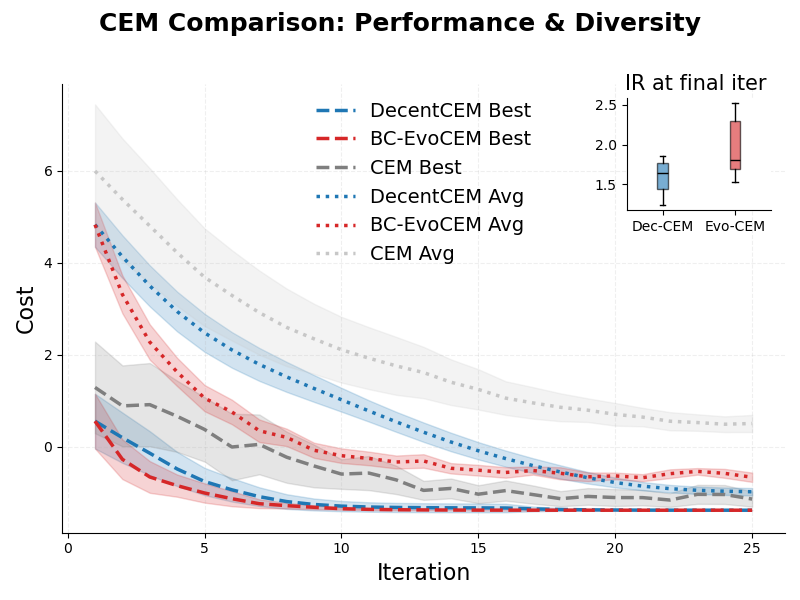} 
\caption{Performance comparison for vanilla, decentralized, and our CEM methods. Solid/dashed lines show the mean/best cost, shaded bands ±1 std. Information‐radius (IR) at iter 25 is shown. }
\label{fig:compar}
\end{wrapfigure}

We first demonstrate our method on a multi-modal optimization problem with the cost function (Fig.~\ref{fig:demo} shows the cost landscape with multiple attraction basins):
\[
    J(\boldsymbol{x})\;=\; \sin(3x_1) + \cos(3x_2) + 0.5\;\|\boldsymbol{x}\|_2^2.
\]
We compare our method against (1) vanilla CEM and (2) decentralized CEM~\cite{zhang2022simple}, with the same parametric distribution $p_{\theta} =\mathcal{N}(\theta, 0.5^2 I)$. Our approach (red in Fig.~\ref{fig:compar}) demonstrates faster convergence in both \emph{Best} and \emph{Average} costs. Importantly, the trust-region sampling maintains solution diversity, as shown by the final IR values (i.e., sample variance in this case).
\clearpage

\newpage
\vspace{-4mm}
\subsection{Navigation Task}\label{sec:planning}
We consider a cluttered 2D point-mass navigation task. Figure~\ref{fig:nav} (left) visualizes trajectories from a fully decentralized CEM, which disperse widely and frequently deviate from the start–goal line. In contrast, \textbf{$\mathcal{BC}$-EvoCEM} (right) maintains a tight cluster of trajectories around the Bregman‐centroid path (green dashed line), producing a more diverse and goal‐directed planning. Notably, the centroid itself is not guaranteed to avoid obstacles as it serves only as an information‐geometric summary of all workers.  Quantitatively, \textbf{$\mathcal{BC}$-EvoCEM} yields significant improvements in both average and best cost without incurring  noticeable computational overhead (see Appendix~\ref{appendix:planning} for a cost summary.)

\begin{figure}[t]
    \centering
    \includegraphics[width=0.9\linewidth]{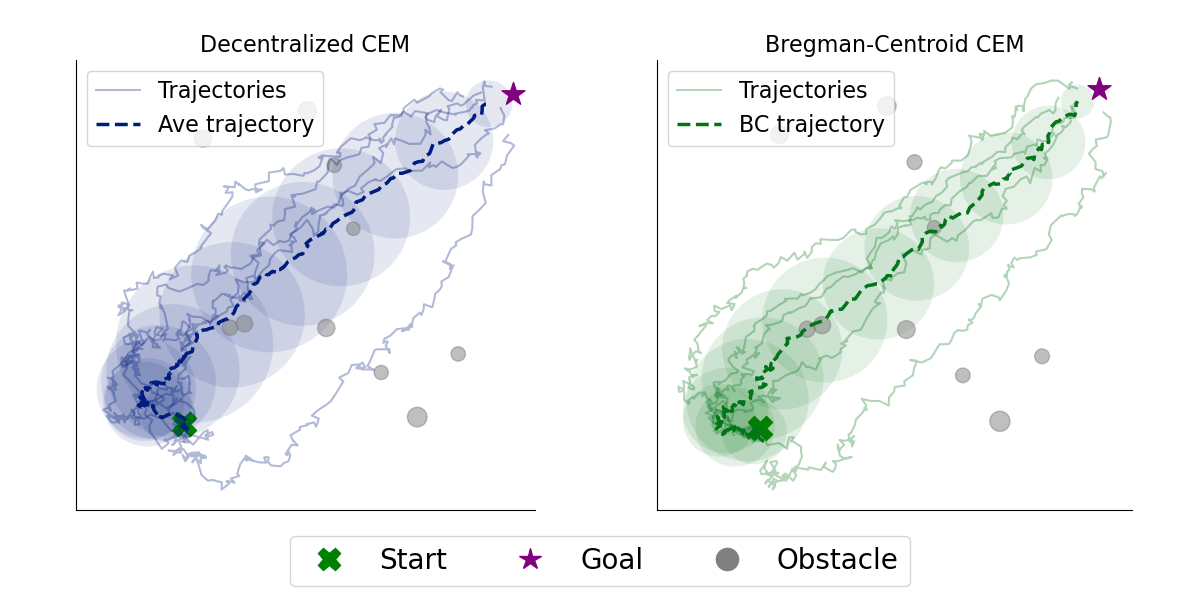}
    \caption{Trajectory distributions from decentralized CEM (left) and Bregman–centroid guided CEM (right) on a point-mass navigation task. The \emph{representatives} of each method (dashed line) are the average and Bregman-centroid trajectory.}
    \label{fig:nav}
\end{figure}

\subsection{Bregman Centroid Guided MPC in MBRL}\label{experiments:mbrl}
\paragraph{Baselines and implementation.}
Our MBRL study builds on the \textsc{PETS} framework~\cite{chua2018deep} and
the \textsc{DecentCEM} implementation~\cite{zhang2022simple}. All components, including dynamics learning, experience replay, and per–worker CEM updates, remain untouched. The proposed \emph{Bregman Centroid Guided MPC} is realized as a simple drop-in wrapper (see Alg.~\ref{alg:bc_mpc_schematic}) for warm starting CEM optimizers. This plug-and-play feature makes the method readily portable to any planning-based MBRL codebase.

\paragraph{Deterministic vs.\ probabilistic ensemble dynamics.}
To isolate the effect of the trajectory optimizer, we fix the PETS baseline and compare our proposed method against both vanilla and decentralized CEM (DecentCEM) under two distinct model classes: 1) a \textbf{deterministic} dynamics model trained by minimizing mean-squared prediction error, and 2) a \textbf{probabilistic ensemble} dynamics with trajectory sampling~\cite{chua2018deep} (full experimental results can be found in Appendix~\ref{appendix:exp_setup}):
\begin{itemize}
\item \emph{Deterministic model.}  
  Our method achieves \textbf{faster learning} and higher
  asymptotic return in most tasks (see Fig.~\ref{fig:train_return_de}). In vanilla CEM the
  sampling covariance collapses rapidly, and in decentralized CEM each
  worker collapses independently.  By contrast, the Bregman
  centroid pulls workers toward promising regions \emph{while the sampling maintains ensemble effective size}. The resulting performance gap therefore quantifies the benefit of injecting \emph{guided} optimality-related randomness during exploration.  
\item \emph{Probabilistic ensemble model.}  
    When both \emph{epistemic} and \emph{aleatoric} uncertainty are captured through model-based trajectory sampling, the performance differences among the three optimizers become statistically indistinguishable (see Fig.~\ref{fig:train_return_pe}). We hypothesize that in such cases, the intrinsic stochasticity of the model induces sufficient trajectory dispersion. Hence, additional optimizer-level
    exploration yields diminishing returns. 
\end{itemize}
\clearpage

\begin{figure}[t]
    \centering
    \includegraphics[width=0.97\linewidth]{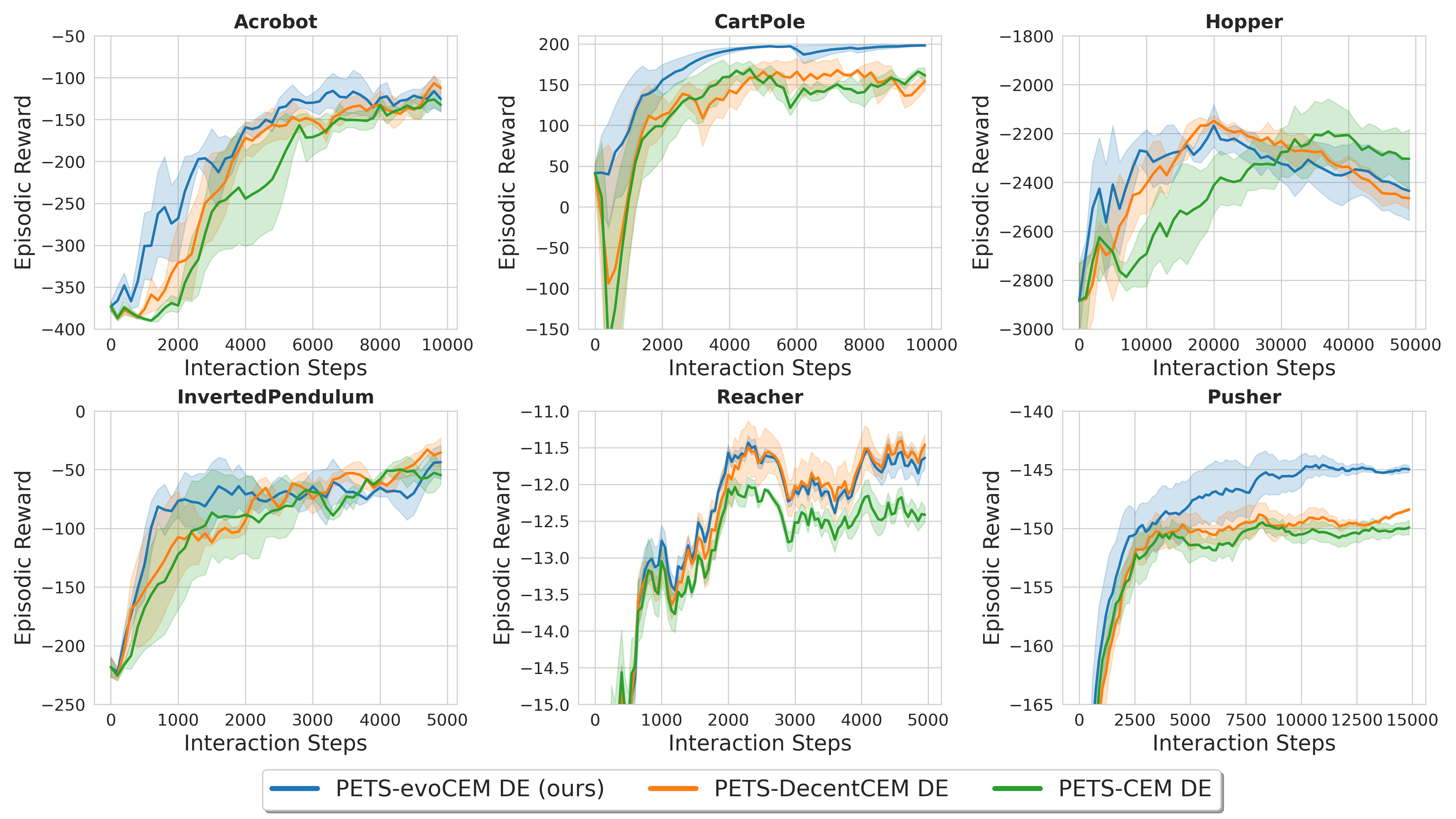}
    \caption{Training return curves across six control tasks using PETS with different CEM-based optimizers. All methods use the \emph{deterministic dynamics model}. Curves show mean performance over 3 random seeds.}
    \label{fig:train_return_de}
\end{figure}
\vspace{-8mm}
\paragraph{Implication on Uncertainties.} The controlled study highlights two distinct yet coupled sources of uncertainty in
model-based RL: \emph{model uncertainty} and \emph{optimality uncertainty}. Improving the dynamics model (e.g., probabilistic ensembles) addresses the former, whereas a diversity-informed optimizer (e.g., \textbf{$\mathcal{BC}$-EvoCEM})  directly addresses the latter. Once the dynamics model approaches its performance cap (or its representational capacity is bottlenecked), optimality uncertainty predominates; in this regime, geometry-informed exploration such as \textbf{$\mathcal{BC}$-EvoCEM} in the action space delivers a complementary boost.

\begin{figure}[H]
    \centering    
    \includegraphics[width=0.97\linewidth]{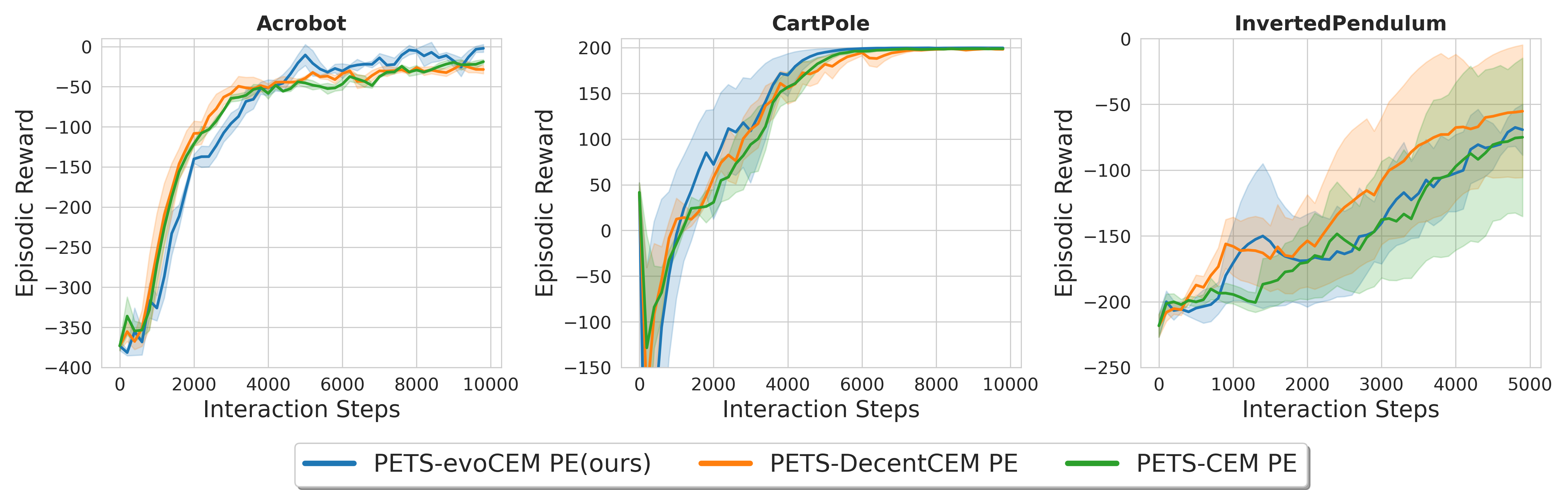}
    \caption{Training return curves across 3 control tasks using PETS with different CEM-based optimizers. All methods use the \emph{probabilistic ensemble dynamics model with trajectory sampling}~\cite{chua2018deep}. Curves show mean performance over 3 random seeds.}
    \label{fig:train_return_pe}
\end{figure}

\section{Conclusion}
\label{sec:conclusion}
We introduced \textbf{$\mathcal{BC}$-EvoCEM}, a lightweight ensemble extension of the Cross-Entropy Method that has (i) principled information aggregation and (ii) diversity-driven exploration with near-zero computation overhead. Across optimization problems and model-based RL benchmarks, \textbf{$\mathcal{BC}$-EvoCEM} demonstrates faster convergence and attains higher-quality solutions than vanilla and decentralized CEM. Its plug-and-play design enables easy integration into MPC loops while preserving the algorithmic simplicity that makes CEM appealing in the first place.


\clearpage
\section*{Limitations}
In this section, we outline several theoretical and empirical limitations of the proposed \textbf{$\mathcal{BC}$-EvoCEM} and provide potential directions for addressing them in future work.

\paragraph{Theoretical Limitations.} All information-geometric arguments (closed-form centroid, ellipsoidal trust region, likelihood-based ranking) hold only for \emph{regular exponential-family} distributions in mean coordinates. This restriction limits the expressiveness of the CEM distributions. Future work will transfer these ideas to richer models via \emph{geometric-preserving} transport maps~\cite{villani2008optimal}. In addition, while we prove the centroid and repawned CEM workers remain inside a Bregman ball, the method still lacks global optimality guarantees and convergence analysis. It inherits these limitations from CEM. A promising direction is to consider the proposed \textbf{$\mathcal{BC}$-EvoCEM} in the stochastic mirror-descent framework~\cite{ahn2021efficient}, which may provide non-asymptotic convergence bounds via primal-dual relationship. 

\paragraph{Empirical Limitations.} All experiments are simulations. Real-time performance of the proposed \textbf{$\mathcal{BC}$-EvoCEM} on real-world robotic platforms remains untested. Future works will deploy \textbf{$\mathcal{BC}$-EvoCEM} on computation-limited hardware to evaluate its performance.

\acknowledgments{
This work was supported in part by NASA ULI (80NSSC22M0070), Air Force Office of Scientific Research (FA9550-21-1-0411), NSF CMMI (2135925), NASA under the Cooperative Agreement 80NSSC20M0229, and NSF SLES (2331878). Marco Caccamo was supported by an Alexander von Humboldt Professorship endowed by the German Federal Ministry of Education and Research.}

\bibliography{example}  

\newpage
\appendix
\section{Relevance Score as Likelihood Evaluation}\label{appendix:scores}
Recall that the Bregman divergence induced by $\Psi$ is
\[
\D_{\Psi}(\theta\parallel\bc)
=\Psi(\theta)-\Psi(\bc)
 -\bigl\langle\nabla\Psi(\bc),\,\theta-\bc\bigr\rangle.
\]
Let $\theta=\theta_{i}$ and denote the dual centroid $\ec=\nabla\Psi(\bc)$.  Expand
\begin{align*}
\gamma_{i}
&= w_{i}\,\D_{\Psi}(\theta_{i}\parallel\bc)\\
&= w_{i}\Bigl[
       \Psi(\theta_{i})
     - \Psi(\bc)
     - \bigl\langle \ec,\,\theta_{i}-\bc\bigr\rangle
   \Bigr]\\
&= w_{i}\Bigl[
       \Psi(\theta_{i})
     - \Psi(\bc)
     - \langle\ec,\,\theta_{i}\rangle
     + \langle\ec,\,\bc\rangle
   \Bigr].
\end{align*}
Since $\Psi(\bc)$ and $\langle\ec,\bc\rangle$ are
independent of~$i$, they are constant across workers and can be dropped
when ranking.  Then, we have
\[
\gamma_{i}
\;\propto\;
w_{i}\Bigl[\Psi(\theta_{i})-\langle\ec,\,\theta_{i}\rangle\Bigr]
\;=\;
-\,w_{i}\,\Bigl[\langle\theta_{i},\,\ec\rangle-\Psi(\theta_{i})\Bigr].
\]
Define the per–sample log-likelihood of the exponential family in
canonical form by
\[
   \ell(\theta;x) \;=\; \langle\theta,x\rangle-\Psi(\theta).
\]
Therefore,
\[
   \gamma_i \;\propto\; -\,w_i\,\ell\bigl(\theta_i;\ec\bigr).
\]

\section{Local Proxy Sampling \& Gaussian Case}\label{appendix:proxy_sampling}
To address the curse of dimensionality in the root solving step in Algorithm~\ref{alg:exact_sampling}, we consider a local approximation of the (dual) trust region 
\[
   \mathcal{S}
   := \{\eta\in\mathcal E : \D_{\Psi^{\!*}}(\ec\parallel\eta)\le\Delta\},
\] 
where \(\Psi^{\!*}\) is the convex conjugate of \(\Psi\). By the definition of the \emph{radial Bregman Divergence} (see~Def.\ref{def:rDiv}), we have \(g_v(0)=0\) and \(\nabla_\rho g_v(0)=0\) at \(\ec\).  A Taylor expansion about \(\rho=0\) gives
\begin{equation}\label{eq:quadratic_proxy}
g_v(\rho)
\;=\;
\frac12\,\rho^2\,v^\top
\underbrace{\nabla^{2}_{\eta}\D_{\Psi^*}\bigl(\ec\|\eta\bigr)\big|_{\eta=\ec}}_{=:\,\mathrm H}
\,v
\;+\;\mathcal O(\rho^3)
\;\approx\;
\frac12\,\rho^2\,v^\top \mathrm H\,v,
\end{equation}
where
\[
\mathrm H
=\nabla^2_\eta\,\D_{\Psi^*}(\ec\|\eta)\big|_{\eta=\ec}
=\nabla^2\Psi^*(\ec)
=\bigl[\nabla^2\Psi(\theta_c)\bigr]^{-1}.
\]
Substituting this quadratic approximation \eqref{eq:quadratic_proxy} into the trust region constraint \(g_v(\rho)\le\Delta\) yields
\[
\frac12\,\rho^2\,v^\top \mathrm H\,v \;\le\;\Delta
\quad\Longrightarrow\quad
\rho \;\le\;
\widehat\rho_{\max}(v)
:=\sqrt{\frac{2\Delta}{v^\top \mathrm H\,v}}.
\]
Hence the proxy trust region in mean space is the Mahalanobis ball
\[
\widehat{\mathcal S}
=\bigl\{\eta:\;(\eta-\ec)^\top \mathrm H\,(\eta-\ec)\le2\Delta\bigr\}.
\]

\paragraph{Diagonal Gaussian Case.}  
Consider the family
\(p_{\theta}(x)=\mathcal N\!\bigl(\mu,\operatorname{diag}(\sigma^{2})\bigr)\)
with natural parameters
\(
\theta_{1i}=\mu_i/\sigma_i^{2},\;
\theta_{2i}=-\tfrac12\sigma_i^{-2}.
\)
Its cumulant function is given by
\[
\Psi(\theta)
=
\sum_{i=1}^{d}
\Bigl[
-\frac{\theta_{1i}^{2}}{4\theta_{2i}}
-\frac12\log(-2\theta_{2i})
+\frac12\log(2\pi)
\Bigr],
\]
and the convex dual in mean coordinates
\(\eta_i=\mu_i\) (fixing \(\sigma_i^{2}\)) is simply
\[
\Psi^{*}(\eta)
=
\frac12\sum_{i=1}^{d}
\frac{(\eta_i-\mu_i)^{2}}{\sigma_i^{2}}
+\text{const}.
\]
Here, the Hessian is
\(
\mathrm H=\nabla^{2}\Psi^{*}(\eta)
      =\operatorname{diag} \!\bigl(\sigma_1^{-2},\dots,\sigma_d^{-2}\bigr),
\)
so the Mahalanobis ball \(\widehat S\) becomes \emph{axis-aligned}:
\[
\Big[\; \boldsymbol{\eta_{c^i}}-\sqrt{2\Delta\,\sigma_i^{2}},\;
      \boldsymbol{\eta_{c^i}}+\sqrt{2\Delta\,\sigma_i^{2}}\; \Big], \quad i=1,\dots,d.
\]
Hence, sampling reduces to independent coordinate draws:
\[
\eta_i
\;\sim\;
\mathrm{Unif}
\Big[\; \boldsymbol{\eta_{c^i}}-\sqrt{2\Delta\,\sigma_i^{2}},\;
      \boldsymbol{\eta_{c^i}}+\sqrt{2\Delta\,\sigma_i^{2}}\; \Big], \quad i=1,\dots,d.
\]

\begin{remark}[On the fixed variance]
During CEM update, the empirical covariance often collapses, becoming low–rank or even singular; in other words, the \emph{mean} component $\mu$ quickly dominates the search directions.  
A practical trick here is to \emph{freeze} the diagonal variance vector $\sigma^{2}$ after a few iterations (or to enforce a fixed lower bound). In practice, we perform such fix-variance trick during the trust region sampling step for high-dimensional planning tasks, including the MPC implementation for MBRL in Sec.~\ref{experiments:mbrl}. 

Because the Hessian matrix (Eq.~\eqref{eq:quadratic_proxy}) is block–diagonal (a diagonal sub‐block for the mean and a sub‐block for the variances), we can safely use the \textsc{Proxy Sampler}~\ref{alg:proxy} to perform such \emph{coordinate‐wise} updates \textbf{exclusively} on the mean block. This also avoids numerical issues from near‐singular covariances.
\end{remark}

\newpage
\section{Proof of Theorem~\ref{thm:sampling_valid}}\label{appendix:proof}
\subsection{Preliminaries}\label{sec:prelim}
Throughout we work on $\Theta,\mathcal E\subset\mathbb R^{d}$ equipped with
Lebesgue measure $\lambda^{d}$.  We write
$\sigma_{d-1}$ for the surface measure on the unit sphere
$\mathbb S^{d-1}:=\{v\in\mathbb R^{d}\mid\|v\|_{2}=1\}$.
The following facts are used (see~\cite{villani2008optimal, schneider2013convex}).

\begin{fact}[Polar coordinates]\label{fact:polar}
Under the polar map
$(\rho,v)\mapsto\eta=\eta_{c}+\rho v$ with $\rho\ge0,\;v\!\in\!\mathbb S^{d-1}$,
the $d$-dimensional Lebesgue volume element factorizes as
$d\eta=\rho^{d-1}\,d\rho\,d\sigma_{d-1}(v)$.
\end{fact}

\begin{fact}[Uniform distribution]\label{fact:uniform}
Let $\rho_{\max}:\mathbb S^{d-1}\!\to\!(0,\infty)$ be measurable and define
\[
\mathcal S
   :=\bigl\{\eta_c+\rho v : v\in\mathbb S^{d-1},\;0\le\rho\le\rho_{\max}(v)\bigr\}.
\]
Then
\[\displaystyle
   \operatorname{Vol}(\mathcal S)
   = \frac1d\!\int_{\mathbb S^{d-1}}\!\rho_{\max}(v)^{d}\,d\sigma_{d-1}(v)
\]
and the uniform law on $\mathcal S$ has a radial conditional density
\[\displaystyle
   f_{\mathcal S}(\rho|v)
   = \frac{d\,\rho^{d-1}}{\rho_{\max}(v)^{d}},
   \qquad 0\leq \rho \leq \rho_{\max}(v).
\]
\end{fact}

\begin{fact}[Change of variables]\label{fact:change}
For $\Psi\!\in\!C^{2}(\Theta)$ strictly convex, the gradient map
$\nabla\Psi:\Theta\!\to\!\mathcal E$ is a $C^{1}$ diffeomorphism
with Jacobian $\det\nabla^{2}\Psi(\theta)$. For any non-negative
$\varphi$,
\[\displaystyle
   \int_{\Theta}\!\varphi(\theta)\,d\theta
   =\int_{\mathcal E}\!\varphi \bigl(\nabla\Psi^{-1}(\eta)\bigr)
     \Bigl|\det\nabla^{2}\Psi\!\bigl(\nabla\Psi^{-1}(\eta)\bigr)\Bigr|\,d\eta.
\]
\end{fact}

\subsection{Auxiliary lemma}
We first show the radial Bregman Divergence (see Def.~\ref{def:rDiv}) is strictly increasing.

\begin{lemma}[Monotonicity]\label{lem:radial_mono}
Let $\Psi^{*}$ be strictly convex and twice differentiable.  For fixed
$\eta_{0}$ and $v\!\in\!\mathbb S^{d-1}$ define
$
   g_v(\rho):=\D_{\Psi^{*}}\!\bigl(\eta_{0}\,\|\,\eta_{0}+\rho v\bigr)
$,
$\rho\ge0$.  Then $g_v$ is strictly increasing on $(0,\infty)$.
\end{lemma}

\begin{proof}
Insert $\eta=\eta_{0}+\rho v$ into $\D_{\Psi^{*}}$ and
differentiate:
$
   g_v'(\rho)
   =\bigl\langle\nabla\Psi^{*}(\eta_{0}+\rho v)-\nabla\Psi^{*}(\eta_{0}),v\bigr\rangle.
$
Strict convexity implies monotonicity of $\nabla\Psi^{*}$; hence
$g_v'(\rho)>0$ for all $\rho>0$.
\end{proof}

\subsection{Main proof}
\begin{theorem*}[Restatement]
Algorithm~\ref{alg:exact_sampling} produces
\(\eta_{\mathrm{new}}\sim\mathrm{Unif}(\mathcal{S})\)
and \(\theta_{\mathrm{new}}\in\mathcal{B}_{\Delta}(\theta_c)\).
If \(\Psi\) is quadratic,
\(\theta_{\mathrm{new}}\) is uniformly distributed in
\(\mathcal{B}_{\Delta}(\theta_c)\).
\end{theorem*}

\begin{proof}
Let $g_v$ be defined as above.

\paragraph{Step 1. Boundary existence \& uniqueness.}
By Lemma~\ref{lem:radial_mono}, $g_v$ is strictly increasing, so
$g_v(\rho)=\Delta$ has a unique root $\rho_{\max}(v)>0$ for each $v$.

\paragraph{Step 2. Feasibility.}
Algorithm~\ref{alg:exact_sampling} draws
$V\sim\operatorname{Unif}(\mathbb S^{d-1})$ and
$U\sim\operatorname{Unif}[0,1]$, sets
$
  \rho = \rho_{\max}(V)\,U^{1/d}
$
and $\eta_{\mathrm{new}}=\eta_{c}+\rho V$.
Because $g_V(\rho)\!\le\!g_V(\rho_{\max}(V))=\Delta$,
$\eta_{\mathrm{new}}\!\in\!\mathcal S$ and hence
$\theta_{\mathrm{new}}:=\nabla\Psi^{-1}(\eta_{\mathrm{new}})\!\in\!\mathcal B_{\Delta}(\theta_c)$.

\paragraph{Step 3. Uniformity.}
Conditioned on $V=v$, $\rho$ has density
$d\,\rho^{d-1}/\rho_{\max}(v)^{d}$ on $[0,\rho_{\max}(v)]$,
which matches Fact~\ref{fact:uniform}; integrating over
$v$ therefore yields $\eta_{\mathrm{new}}\sim\operatorname{Unif}(\mathcal S)$.

\paragraph{Step 4. Pull-back to $\Theta$.}
By Fact~\ref{fact:change},
$
   f_{\theta}(\theta)
   = f_{\mathcal S}\bigl(\nabla\Psi(\theta)\bigr)\,
     \bigl|\det\nabla^{2}\Psi(\theta)\bigr|.
$
For general $\Psi$, $\det\nabla^{2}\Psi(\theta)$ varies with $\theta$,
so $f_{\theta}$ is not constant.  If $\Psi$ is quadratic,
$\nabla^{2}\Psi$ is constant; hence $f_{\theta}$ is constant on
$\mathcal B_{\Delta}(\theta_c)$, i.e.\ $\theta_{\mathrm{new}}$ is uniform.   
\end{proof}

\newpage
\section{Experimental Details}\label{appendix:exp_setup}
\subsection{Navigation Task}\label{appendix:planning}
We consider a cluttered 2D navigation task with first‐order dynamics and time‐step \(\Delta t = 0.2\).  A planning horizon of \(H = 200\) yields a \(2H\)-dimensional action sequence.  We employ 5 independent diagonal‐Gaussian CEM workers with identical CEM hyperparameters and initialization.  To sample from the trust region in this high‐dimensional space, we use the \textsc{ProxySampler} (Alg.~\ref{alg:proxy}).

\begin{table}[ht]
  \centering
  \caption{Normalized costs and relative drop versus decentralized CEM.}
  \label{tab:normalized-costs}
  \begin{tabular}{l  
                  cc  cc}
    \toprule
    & \multicolumn{2}{c}{Average cost} 
    & \multicolumn{2}{c}{Best cost} \\
    \cmidrule(lr){2-3}\cmidrule(lr){4-5}
    Method                   & Norm. & Drop\,(\%) 
                             & Norm. & Drop\,(\%) \\
    \midrule
    Decentralized CEM        & 1.00               & —           
                             & 1.00               & —          \\
    Bregman–Centroid CEM     & 0.18               & 82.4        
                             & 0.55               & 45.3       \\
    \bottomrule
  \end{tabular}
\end{table}

\subsection{MBRL Benchmark}
\subsubsection{Benchmark Environment Setup} 
We follow the evaluation protocol of \cite{zhang2022simple} to assess both our method and the baseline algorithms on the suite of robotic benchmarks introduced by \cite{wang2019benchmarking,chua2018deep}, including classical robotic control problems and high-dimensional locomotion and manipulation problems. Key environment parameters are summarized in Table~\ref{tab:env_parameter}. We refer the interested readers to \cite{zhang2022simple} for more details, such as reward function settings, termination conditions, and other implementation specifics. For each case study, all algorithms are trained on three random seeds and evaluated on one unseen seed. 

\begin{table}[h]
    \centering
    \renewcommand{\arraystretch}{1.2}
    \caption{Details of Benchmark Environments}
    \label{tab:env_parameter}
    \begin{tabular}{lcccccc}
        \hline
        \textbf{Parameter} & \textbf{Acrobot} & \textbf{CartPole} & \textbf{Hopper} & \textbf{Pendulum} & \textbf{Reacher} & \textbf{Pusher} \\
        \hline
        Train Iterations  & 50 & 50 & 50 & 50 & 100 & 100 \\
        Task Horizon  & 200 & 200 & 1000 & 100 & 50 &  150\\
        Train Seeds  & $\{1, 2, 3\}$ & $\{1, 2,3\}$ & $\{1, 2,3\}$ & $\{1, 2,3\}$ & $\{1, 2,3\}$ & $\{1, 2,3\}$ \\
        Test Seeds  & $\{0\}$ & $\{0\}$  & $\{0\}$  & $\{0\}$  & $\{0\}$  & $\{0\}$  \\
        Epochs per Test & 3 & 3 & 3 & 3 & 3 & 3 \\
        \hline
    \end{tabular}
\end{table}

\subsubsection{Algorithms Setup}
The key parameters for the proposed \textbf{$\mathcal{BC}$-EvoCEM} algorithm and all baseline methods are listed in Table~\ref{tab:algo_parameter}. The dynamic model for each benchmark is parameterized as a fully connected neural network: four hidden layers with 200 units each, except for the \emph{Pusher} task, which uses three hidden layers. All algorithms share identical training settings for learning the dynamics model; further details on model learning can be found in \cite{chua2018deep} and \cite{zhang2022simple}.

\begin{table}[h]
    \centering
    \renewcommand{\arraystretch}{1.2}
    \caption{Details of Algorithms {(\textbf{DE} and \textbf{PE}})}
    \label{tab:algo_parameter}
    \begin{tabular}{lccc}
        \hline
        \textbf{Parameter} & \textbf{PETS-evoCEM} & \textbf{PETS-DecentCEM} &  \textbf{PETS-CEM}   \\
        \hline
        CEM Type  & $\mathcal{BC}$-EvoCEM & DecentCEM & CEM \\
        CEM Ensemble Size  & 3 & 3 & 1 \\
        CEM Population Size & 100 & 100 & 100\\
        CEM Proportion of Elites & 10 \% &  10 \% &  10 \% \\
        CEM Initial Variance & 0.1 & 0.1 & 0.1 \\
        CEM Internal Iterations & 5 & 5 & 5\\
        Model Learning Rate & 0.001 & 0.001 & 0.001 \\
        Warm-up Episodes & 1 & 1 & 1 \\
        Planning Horizon & 30 & 30 & 30 \\
        \hline
    \end{tabular}
\end{table}

\subsubsection{Full Experimental Results.}
\begin{figure}[H]
    \centering
    \includegraphics[width=0.95\linewidth]{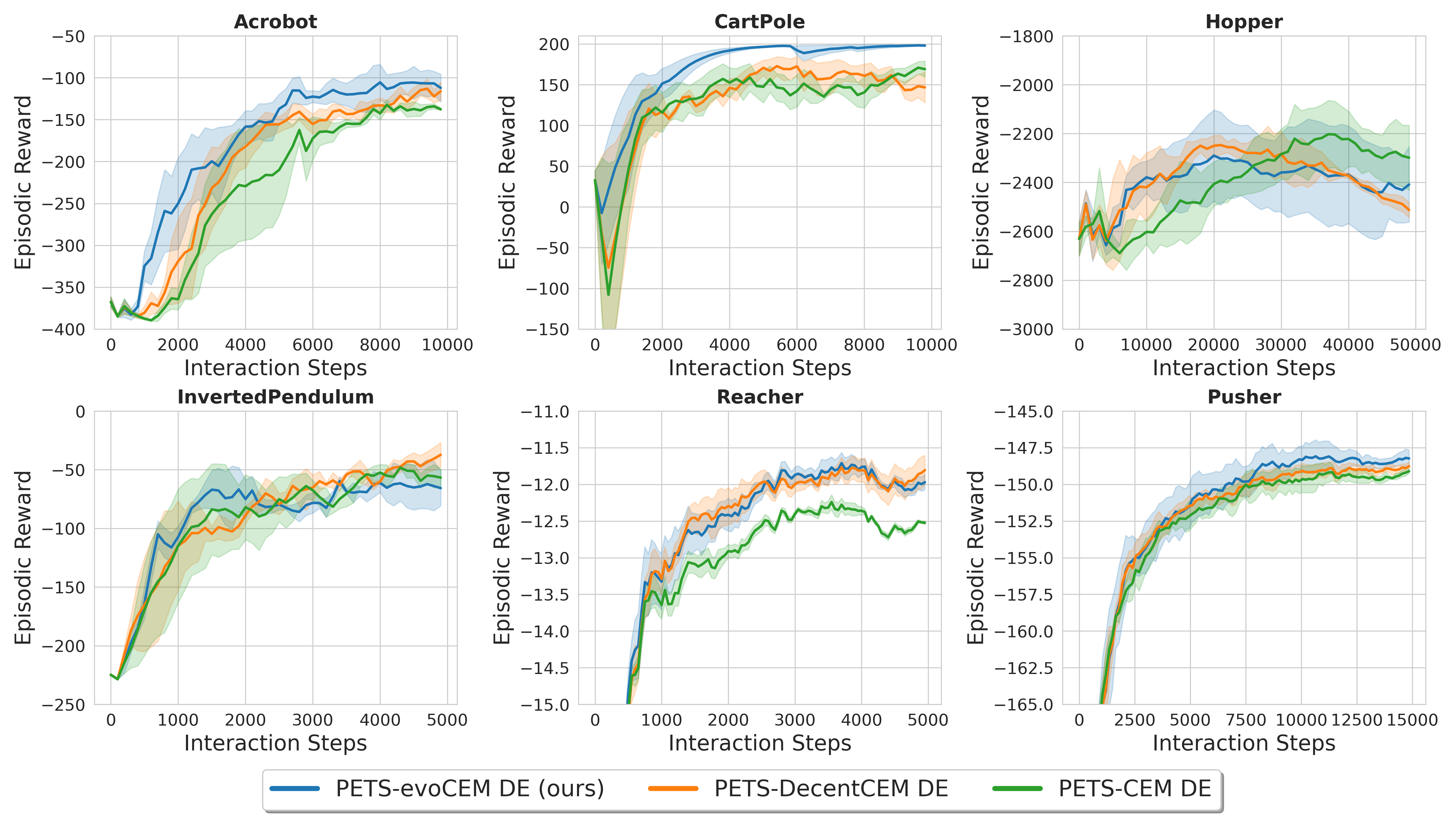}
    \caption{Testing return curves across six control tasks using PETS with different CEM-based optimizers. All methods use the \emph{deterministic dynamics model}. Curves show mean performance over 3 random seeds.}
    \label{fig:test_return_de}
\end{figure}

\begin{figure}[H]
    \centering    
    \includegraphics[width=0.95\linewidth]{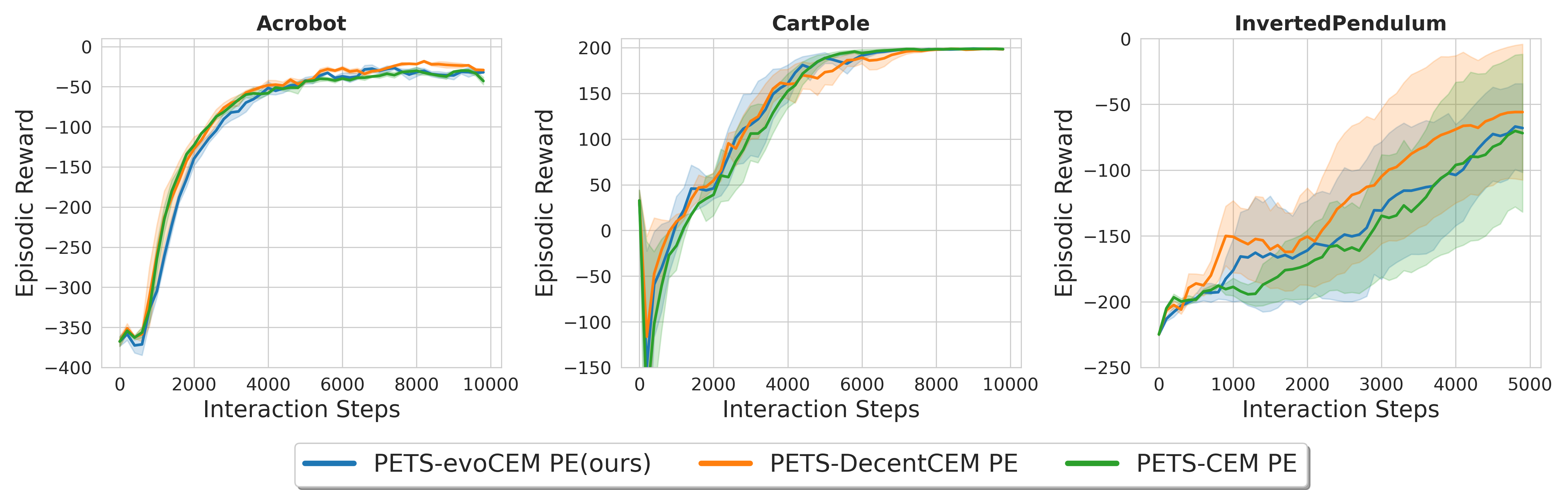}
    \caption{Testing return curves across 3 control tasks using PETS with different CEM-based optimizers. All methods use the \emph{probabilistic ensemble dynamics model with trajectory sampling}~\cite{chua2018deep}. Curves show mean performance over 3 random seeds.}
    \label{fig:test_return_pe}
\end{figure}

\end{document}